\documentclass[11pt,letterpaper]{article}
\usepackage{fullpage}
\usepackage{mathpazo}
\usepackage{caption}

\usepackage[utf8]{inputenc} %
\usepackage[T1]{fontenc}    %
\usepackage{hyperref}       %
\usepackage{url}            %
\usepackage{booktabs}       %
\usepackage{amsfonts}       %
\usepackage{nicefrac}       %
\usepackage{microtype}      %
\usepackage{caption}
\usepackage{amssymb}
\usepackage{graphicx}
\usepackage{color}
\usepackage{complexity}
\usepackage{subcaption}
\usepackage{amsmath}
\usepackage{amsthm}
\usepackage{pifont}
\usepackage{tabularx}
\usepackage{fancyvrb}
\usepackage{comment}
\usepackage{algorithm}
\usepackage{algorithmic}
\usepackage{float}
\usepackage{wrapfig}
\usepackage{mathtools}
\usepackage{thmtools,thm-restate}
\usepackage{paralist} 
\usepackage{multirow}
\usepackage{scrextend}
\usepackage{enumitem}
\usepackage[noorphans,vskip=0ex]{quoting}
\usepackage{bm}
\usepackage{bbm}
\usepackage{thm-restate}
\usepackage{tabularx}
\usepackage{xcolor}
\usepackage{xspace}
\usepackage{adjustbox}
\usepackage{textcomp}
\usepackage{pgf,tikz,pgfplots}
\pgfplotsset{compat=1.15}
\usepackage{mathrsfs}
\usetikzlibrary{arrows}
\usepackage[round]{natbib}
\usepackage{times}
\setlist{nolistsep}
\theoremstyle{definition}

\DeclareMathOperator*{\argmin}{arg\,min}

\DeclarePairedDelimiterX{\inp}[2]{\langle}{\rangle}{#1, #2}

\def\R{\mathbb{R}}

\def\mR{\mathcal{R}}

\def\mG{\mathcal{G}}

\def\mF{\mathcal{F}}

\def\mS{\mathcal{S}}
\def\fl{\mF_{\lambda}}

\newcommand{\mypara}[1]{{\smallskip \noindent \bf #1}\hspace{0.1in}}
\definecolor{dkgreen}{rgb}{0,0.6,0}
\definecolor{gray}{rgb}{0.5,0.5,0.5}
\definecolor{mauve}{rgb}{0.58,0,0.82}
\setlist{nolistsep}

\title{Personalized Federated Learning with Attention-based Client Selection}

\author{Zihan Chen, Jundong Li, Cong Shen}%\thanks{Thanks to NSF for funding.}}
\date{University of Virginia}

\begin{document}

\maketitle

\begin{abstract}
Personalized Federated Learning (PFL) relies on collective data knowledge to build customized models. However, non-IID data between clients poses significant challenges, as collaborating with clients who have diverse data distributions can harm local model performance, especially with limited training data. To address this issue, we propose \textsc{FedACS}, a new PFL algorithm with an \underline{A}ttention-based \underline{C}lient \underline{S}election mechanism. \textsc{FedACS} integrates an attention mechanism to enhance collaboration among clients with similar data distributions and mitigate the data scarcity issue. It prioritizes and allocates resources based on data similarity. We further establish the theoretical convergence behavior of \textsc{FedACS}. Experiments on CIFAR10 and FMNIST validate \textsc{FedACS}'s superiority, showcasing its potential to advance personalized federated learning. By tackling non-IID data challenges and data scarcity, \textsc{FedACS} offers promising advances in the field of personalized federated learning.
\end{abstract}

\section{Introduction}

Federated learning is a collaborative learning paradigm that allows multiple clients to work together while ensuring the preservation of their privacy \cite{yang2019federated}. By leveraging the collective knowledge and data from all participating clients, federated learning aims to achieve better learning performance compared with individual client efforts \cite{mcmahan2017communication}. This collaborative nature has made federated learning increasingly popular, finding numerous practical applications where data decentralization and privacy are paramount. The privacy-preserving solutions offered by federated learning have found extensive applications in domains such as healthcare, smart cities, and finance \cite{zheng2022applications,yang2019ffd,xu2021federated}.

However, the effectiveness of the collaborative approach in federated learning is highly dependent on the distribution of data among the clients. While federated learning performs exceptionally well when data distribution among clients is independent and identically distributed (IID), this is not the case in many real-world scenarios \cite{kairouz2021advances}. When the global model, which is collectively trained across decentralized clients, encounters diverse datasets with varying statistical characteristics, it may face challenges in effectively generalizing to the unique local data of each client \cite{zhao2018federated,jiang2019improving}. The performance of the global model may be suboptimal on certain clients' data due to differences in data distributions and patterns. This limitation becomes more pronounced as the diversity among local data from different clients continues to increase \cite{deng2020adaptive}.

To address this issue, personalized federated learning (PFL) techniques have been proposed, which aim to overcome the limitations of traditional federated learning by incorporating personalized adaptation in local clients. Instead of relying solely on a single global model, PFL allows each client to tailor the model to her specific data distribution, preferences, and local context \cite{kulkarni2020survey}. Most existing PFL methods focus on taking the global model as an initial model for local training or incorporating the global model into the loss function \cite{deng2020adaptive,li2021ditto,t2020personalized,fallah2020personalized}. However, recent work by Huang et al. \cite{huang2021personalized} challenges the assumption that a single global model can adequately fit all clients in personalized cross-silo federated learning with non-IID data. In particular, they argue that the misassumption of a universally applicable global model remains a fundamental bottleneck \cite{huang2021personalized}. However, their claims lack substantial theoretical foundations or sufficient empirical evidence. To shed light on this issue, Mansour et al. \cite{mansour2020three} provide theoretical insights into the generalization of the uniform global model and demonstrate that the discrepancy between local and global distributions influences the disparity between local and global models. In this study, we present an empirical example to further validate the idea that a single global model cannot adequately fit all clients. Table \ref{tab:result_intro} presents the learning performance of local training models and two PFL methods \cite{fallah2020personalized,li2021ditto}. For each dataset, we utilize a Dirichlet distribution with the parameter of 0.5 to partition it, and subsequently, we sample 50 data points from the training dataset for the training process.  Specifically, \textsc{Ditto-P} and \textsc{PerAvg-P} represent the results of personalized models, while \textsc{Ditto-G} and \textsc{PerAvg-G} reflect the performance of global models generated by the respective algorithms. It is evident from the table that the global models produced by the PFL methods struggle to achieve satisfactory generalization across clients and, in some cases, even perform worse than the locally trained models. This empirical example highlights the need for novel approaches that go beyond the simple fine-tuning of global models and account for the inherent heterogeneity in data distributions among clients.
\begin{table}
	\caption{Test accuracy of global and personalized models.}
	\centering
	\begin{tabular}{l|cc}
		\toprule

		Dataset     &CIFAR10 dir = 0.5     &FMNIST dir = 0.5 \\
		\midrule
		Local  &77.29  &75.98   \\
        \textsc{Ditto-P} &77.33	&78.79	\\
        \textsc{Ditto-G} &39.47	&76.34	\\
        \textsc{PerAvg-P} &73.07	&74.07	\\
        \textsc{PerAvg-G} &28.90	&69.14	\\
		\bottomrule
	\end{tabular}
	\label{tab:result_intro}
\end{table}

Motivated by the insights gained from \cite{huang2021personalized} and recognizing the challenges posed by the scarcity of data in specific areas such as medical care, in this paper, we propose a PFL algorithm with an \underline{A}ttention-based \underline{C}lient \underline{S}election mechanism (\textsc{FedACS}) that aims to overcome the dependency on a single global model and leverage the model information from other clients to develop a personalized model for each client. By employing an attention mechanism, \textsc{FedACS} allows clients to receive and process personalized messages tailored to their specific model parameters, enhancing the effectiveness of collaboration. Additionally, we provide the convergence of our proposed \textsc{FedACS} method in general scenarios, ensuring the reliability and stability of the algorithm. This demonstrates that the proposed method can effectively converge to optimal solutions for personalized models in general settings. To evaluate the effectiveness of the proposed methods, we conduct extensive experiments on various datasets and settings, allowing for a comprehensive performance evaluation. The results of the experiments demonstrate the superior performance of the proposed methods compared to those of existing approaches. 

\section{Personalized Federated Learning}
\label{sec:method}
\subsection{Formulation}
Standard federated learning (FL) aims to train a global model by aggregating local models contributed by multiple clients without the need for their raw data to be centralized. In FL, there are $n$ clients communicating with a server to solve the problem:
$\min_{w} \frac{1}{n}\sum_{i=1}^{n} \mathbb{F}_{i}(w)$ to find a global model $w$. The function $\mathbb{F}_{i}$ denotes the expected loss over the data distribution of the client $i$. One of the first FL algorithms is \textsc{FedAvg} \cite{mcmahan2017communication}, which uses local stochastic gradient descent (SGD) updates and
builds a global model from a subset of clients. 

However, the performance of the global model tends to degrade when faced with clients whose data distributions differ significantly from the global training data distribution. On the other hand, local models trained on the respective data distributions match the distributions at inference time, but their generalization capabilities are limited due to the data scarcity issue \cite{mansour2020three}. Personalized federated learning (PFL) can be viewed as a middle ground between pure local models and global models, as it combines the generalization properties of the global model with the distribution matching characteristic of the local model \cite{li2021ditto,fallah2020personalized,t2020personalized,deng2020adaptive}.

In PFL, the goal is to train personalized models $w_1, w_2,..., w_n$ for individual clients while respecting data privacy and accommodating user-specific requirements. The problem can be framed as: 
\[w_{1}^{*}, w_{2}^{*}, ..., w_{n}^{*}= \argmin_{w_1,..., w_n} \sum_{i=1}^{n} \mathbb{F}_{i}(w_i).\]

Instead of solving the traditional PFL problem, Huang et al. \cite{huang2021personalized} argued that the fundamental bottleneck in personalized cross-silo federated learning with non-IID data is that the misassumption of one global model can fit all clients. Therefore, they proposed a different approach by adding a regularized term with $\ell_2$-norm of model weights distances among clients and formulate the PFL problem as:
\begin{equation}
\min_W \fl(W)=\mF(W)+\lambda\mR(W)=\mF(W)+\lambda\sum_{i,j=1}^{n} \mR(||w_i-w_j||^2), \label{optim}
\end{equation}
in which $W = [w_1, . . . , w_n]$ is a matrix that contains clients' local models as its columns and $\lambda$ is a regularization parameter. $\mF(W)=\sum_{i=1}^{n} \mathbb{F}_{i}(w_i)$ is known as ``client-side personalization" in the context of PFL. $\mR$ is a regularization term designed to help clients with similar data distributions collaboratively train their own models.

\subsection{Our proposed method: \textsc{FedACS}}

We propose a novel federated attention-based client selection algorithm (\textsc{FedACS}) (as illustrated in Algorithm\ref{algorithm}) to solve the optimization problem \eqref{optim}.
In this paper, we mainly consider:
\begin{equation}
    \mR(W)=\sum_{i,j=1}^{n} s_{ij}||w_i-w_j||^2, \label{R}
\end{equation}
where $s_{ij}$ is the normalized score to measure data distribution similarity. It is worth mentioning that the formation of $\mR(W)$ is easy to generate into other functions, like $\mR(x)=1 - e^{-x^{2}/ \sigma}$ \cite{huang2021personalized}. The above formulation implies that locally optimizing the objective function requires each client to collect the other clients' model parameters for computation. However, researchers find that adversaries can infer data information in the training set based on the model parameters, and directly gathering model parameters will violate the principles of federated learning to protect data privacy \cite{fredrikson2015model,shokri2017membership,hitaj2017deep}. To address this dilemma, we leverage the idea of incremental optimization \cite{bertsekas2011incremental}. Specifically, we iteratively optimize $\fl(W)$ by alternatively optimizing $\mR(W)$ and $\mF(W)$ until convergence. Note that \textsc{FedACS} handles the refinement of $\mR(W)$ on the server side. We introduce an intermediate model $u_i$ for each client, and these intermediate models form the model matrix $ U = [u_1,. . , u_n]$ in the same way as $W$. In the $k$-th iteration, we first apply a gradient descent step to update $U$:
\begin{equation}
    U^k=W^{k-1}-\alpha_{k} \nabla \mR(W^{k-1}), \label{Uup}
\end{equation}
where $W^{k-1}$ is the collection of local models from the last round and $\alpha_{k}$ is the learning rate. Under the definition \eqref{R}, the gradient of $\mR(W)$'s $i$-th column is $\nabla \mR(W)_i = w_{i}-\sum_j s_{ij}w_{j}$. In this paper, we use the cosine similarity of the model parameters instead of the data similarity to avoid the accessibility of the server to local data, i.e. $s_{ij}=\frac{\langle w_i, w_j \rangle}{||w_i||||w_j||}$. 

As mentioned above, the performance of a local model could be adversely affected if each client averages information from those with different data distributions. Thus, in practice, we also introduce $\delta$ to control the degree of collaboration.
We adopt a dynamic strategy to define $\delta$ for each round. After gathering the parameters from clients at the $k$-th round, the server first computes the similarity of the model $s_{ij}^{k}$ to generate the similarity matrix $\mS^{k}:(\mS^{k})_{ij}=s_{ij}^{k}$, then $\delta$ is calculated by the $p$-quantile of all elements in $\mS^{k}$. We can adjust the ratio $p$ to alter the degree of collaboration. Modifying $\delta$ can effectively mitigate the aggregation effect that arises from initial training rounds, where the model parameters fail to accurately represent the underlying data distribution accurately.
After having $\delta^{k}$, we only consider the model similarity greater than $\delta^{k}$ and calculate $u_i^{k}$ as follows:

\begin{equation}
    u_{i}^{k}=\frac{1}{\sum_{j=1,s_{ij}^{k}>\delta^{k}}^{n}s_{ij}^{k}} \sum_j \textbf{I}_{s_{ij}^{k}>\delta^{k}}s_{ij}^{k} w_j^{k-1}.\label{ui}
\end{equation}

Since the similarity of the client model to itself is always 1 and greater than $\delta^{k}$. Like the attention mechanism in deep learning, clients with highly similar data distributions receive higher similarity scores, enabling the server to assign greater weights to their models during the combination \cite{vaswani2017attention}. By differentiating between models based on their relevance and informativeness, the server can selectively assign the most pertinent models to each client rather than treating all models equally. Furthermore, the server can collect weights from clients and utilize the attention mechanism to optimize $\mR(W)$ effectively while still ensuring the preservation of data privacy for all the involved clients.
% \begin{figure}[htb]
% \begin{minipage}[b]{1.0\linewidth}
%   \centering
%   \centerline{\includegraphics[width=8.5cm]{framework.png}}
% \end{minipage}
% \caption{The framework of \textsc{FedACS}}
% \label{fig:\textsc{FedACS}}
% \end{figure}

After optimizing $\mR(W)$  on the server, \textsc{FedACS} then optimizes $\mF(W)$ on clients' devices by taking $U^{k}$ as initialization of $W^{k}$ and computes $w_{i}^{k}$ locally by:
\begin{equation}
    w_{i}^{k}= u_{i}^{k}-\beta_{k} \nabla \mathbb{F}_{i}(u_{i}^{k}).\label{wi}
\end{equation}
The update of $w_{i}^{k}$ draws inspiration from local fine-tuning in FL that takes the global model as initialization and updates the local models with several gradient steps to fit local data distribution. Considering the scarcity of data, it is more efficient to use $u_{i}^{k}$ as an initial value instead of incorporating it into the objective function. . This research is driven by prior studies that have shown difficulties in achieving high training accuracy when the data is either limited or unevenly spread out, thus highlighting the significance of a carefully selected starting point\cite{li2019convergence}.

By iteratively optimizing $\fl(W)$ through alternating between $\mR(W)$ and $\mF(W)$ optimization, \textsc{FedACS} continues this process until a predefined maximum number of iterations, $K$, is reached. 

\begin{algorithm}[t]
 			\caption{The proposed  \textsc{FedACS} method}
			\label{algorithm}
			\begin{algorithmic}[1]
				\STATE \textbf{Input}: $n$ clients, pick ratio $p$ , communication round $K$
                    \STATE \textbf{Output}: Personalized models $w_1,w_2,...,w_n$
                    \STATE Initialize clients' models $w_{1}^{0},w_{2}^{0},...,w_{n}^{0}$;
				\FOR{round $k=1,2,...,K$}
			        \STATE Server randomly picks clients to participate;
           			\STATE Compute model similarity matrix $\mS$;
                        \STATE Compute  $p$-quantile of model similarities as threshold $\delta$;
                        \STATE Server updates intermediate models $u_{1}^{k},u_{2}^{k},...,u_{n}^{k}$ for clients  by \eqref{ui};
                        \STATE Each client $i$ updates local model $w_{i}^{k}$ by \eqref{wi};
				\ENDFOR
			\end{algorithmic}
\end{algorithm}

\subsection{Relation to FedAMP}
We note that the usage of the regularization term in \textsc{FedACS} bears resemblance to \textsc{FedAMP}\cite{huang2021personalized}, a method that also applies attentive message passing to facilitate similar clients to collaborate more. However, \textsc{FedACS} differs from \textsc{FedAMP} in the update steps. First, when computing the intermediate model matrix $U$, \textsc{FedAMP} combines all client models for each client, while \textsc{FedACS} only uses information from clients who share similar data distribution to further reduce the influence from unrelated clients. Meanwhile, if we set the threshold $\delta = \infty$ then the calculation of $U$ is the same as that in \textsc{FedAMP}. Second, for each round, \textsc{FedAMP} updates $W^k$ by $W^k = \argmin_{W} \mF(W)+\frac{\lambda}{2\alpha_{k}}||W-U^{k}||^{2}$, which means the $W^k$ needs to reach the local optimal. However, it is not easy to be satisfied when data is insufficient on local devices, and when the learning rate is small enough, the update function in \textsc{FedAMP} will be simplified to $W^k = U^{k}$ without using local data information. Additionally, the Euclidean distance version of \textsc{FedAMP} incorporates an $\ell_2$-term regularization, which adds the Euclidean distance of the models to the objective function without using weighted scores. However, this approach provides limited assistance to the server in effectively assigning similar clients' models to each client, thereby underutilizing the potential of the attention mechanism. As discussed above, the limitation of data may adversely influence the performance of \textsc{FedAMP}, while our proposed \textsc{FedACS} will not suffer from this issue. The experiment results also demonstrate our conjectures.

\subsection{Convergence analysis of \textsc{FedACS}}
In this subsection, we provide the convergence analysis of \textsc{FedACS} under suitable conditions without the special requirement of the convexity of the objective function. We start with the definition of $\epsilon$-approximate first-order stationary point.

\mypara{Definition 1.} A point $W$ is considered an $\epsilon$-approximate First-Order Stationary Point (FOSP) for the problem \eqref{optim} if it satisfies $||\nabla \fl(W)||\leq \epsilon.$

Consistent with the analysis performed in various incremental and stochastic optimization algorithms \cite{bertsekas2011incremental,nemirovski2009robust,huang2021personalized}, we introduce the following assumption and present the main result of the paper.

\mypara{Assumption 1.} There exists a constant $B > 0$ such that $max \{||Y||: Y \in \partial \mF(W^k)\} \leq B$ and $||\nabla \mR(W_k)|| \leq B/\lambda$ hold for every $k \geq 0$, where $ \partial \mF$ is the subdifferential of $\mF$ and $||\cdot ||$ is the Frobenius norm.

\begin{restatable}{theorem}{maintheory}
Assuming the validity of Assumption 1 and considering the continuous differentiability of functions $\mR(W)$ and $\mF(W)$, with gradients that exhibit Lipschitz continuity with a modulus $L$, if $\beta_1=\beta_2=\dots=\beta_K=1\sqrt{K}$ and $\alpha_1=\alpha_2=\dots=\alpha_K=\lambda/\sqrt{K}$, then the model matrix sequence $W^0, ..., W^K$ generated by Algorithm 1 satisfies the following.
\begin{equation}
    \min_{0\leq k \leq K}||\nabla \fl(W^{k})||^2\leq \frac{6( \fl(W^{0})-\fl^{*}+16LB^2)}{\sqrt{K}}+\mathcal{O}\bigg( \frac{1}{K}\bigg),
\end{equation}
where $\fl^{*} = \argmin_{W}\fl^{W}$. Additionally, if $\beta_k = \alpha_k/\lambda$ and $\{\alpha_k\}$ satisfies $\sum^{\infty}_{k=1} \alpha_{k}=\infty$ and $\sum^{\infty}_{k=1} \alpha_{k}^2<\infty$, then
\begin{equation}
    \liminf_{k\rightarrow \infty}||\nabla \fl(W^{k})||=0.
\end{equation}

\label{theory:main}
\end{restatable}

\mypara{Remark.} Theorem 1 implies that for any $\epsilon > 0$, \textsc{FedACS} needs at most $\mathcal{O}(\epsilon^{-4})$ iterations to find an $\epsilon$-approximate stationary point $\hat{W}$ of problem \eqref{optim} such that $||\nabla \fl(\hat{W})||\leq \epsilon$. It also
establishes the global convergence of \textsc{FedACS} to a stationary point of problem \eqref{optim} when $\fl$ is smooth and non-convex. We are unable to provide proof due to page limitations.

\section{Experiments}

\subsection{Experimental settings}
We evaluate the performance of \textsc{FedACS} and compare it with the state-of-the-art PFL algorithms, including \textsc{Ditto} \cite{li2021ditto}, \textsc{PerAvg} \cite{fallah2020personalized}, \textsc{pFedMe} \cite{t2020personalized}, \textsc{APFL} \cite{deng2020adaptive}, and \textsc{FedAMP} \cite{huang2021personalized}. The first three methods are local fine-tuning methods, using the global model to help fine-tune local models. \textsc{FedAMP} is a method that focuses on adaptively facilitating pairwise collaborations among clients. To ensure the completeness of our experiments, we have also included the results obtained using the standard FL method \textsc{FedAvg} \cite{mcmahan2017communication}. The performance of all methods is evaluated by the mean test accuracy across all clients. We perform five experiments for each dataset and partition configuration, wherein we record both the mean and variance of the test accuracy. All methods are implemented in PyTorch 1.8 running on NVIDIA 3090.

We use two public benchmark data sets, FMNIST (FashionMNIST) \cite{xiao2017fashion} and CIFAR10 \cite{krizhevsky2009learning}.
First, we simulate the non-IID and data-sufficient settings by employing two classical data split methods to distribute the data among 100 clients. These methods include: 1) a pathological non-IID data setting, where the dataset is partitioned in a non-IID manner, assigning two classes of samples to each client; and 2) a practical non-IID data setting, where the dataset is partitioned in a non-IID manner following a Dirichlet 0.5 distribution \cite{mcmahan2017communication}. To simulate scenarios where data insufficiency is encountered during training for each dataset, we adopt two distinct data settings: 1) a pathological non-IID data setting that distributes data into 100 clients following Dirichlet 0.5 distribution, and each user only keeps 50 pieces of data samples for training to simulate the scenario where clients lack sufficient training data: 2) a practical non-IID data setting that distributes data into 500 clients following Dirichlet 0.5 distribution to simulate the scenario of a large number of clients and only some clients participate in each round. 

\begin{figure*} [h]
    \centering
    \begin{subfigure}[b]{0.45\textwidth}
        \includegraphics[width=\textwidth,height=0.45\textwidth]{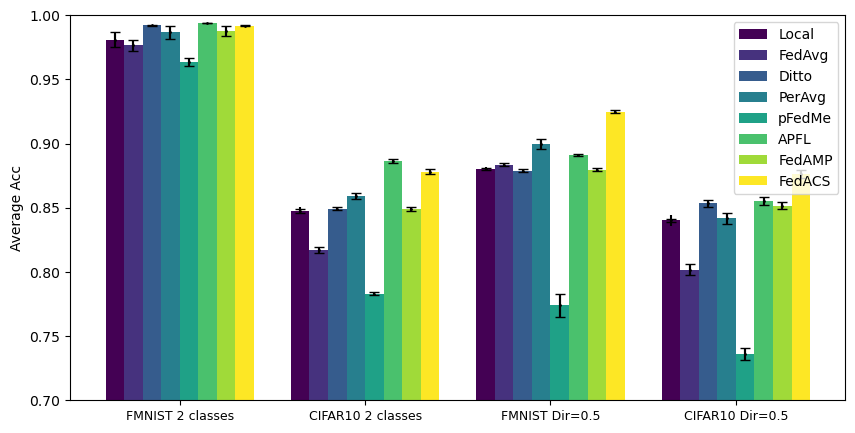}
        \caption{Performance on tasks with sufficient data.}
        \label{fig:suf}

    \end{subfigure}
    \hfill
    \begin{subfigure}[b]{0.45\textwidth}
        \includegraphics[width=\textwidth,height=0.45\textwidth]{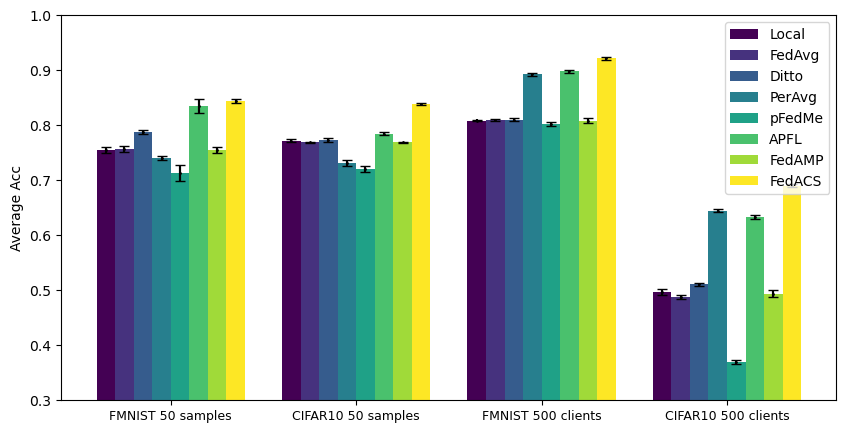}
        \caption{Performance on tasks with insufficient data.}
        \label{fig:insuf}
    \end{subfigure}
    \caption{Experiment results}
    \label{fig:main}
\end{figure*}
\subsection{Results on the sufficient non-IID data setting}
Figure \ref{fig:suf} presents the mean accuracy and standard deviation per method and sufficient data setting. The two histograms on the left display the performance of PFL methods under pathological settings, while the right-side histograms illustrate the results when the data is partitioned according to the Dirichlet 0.5 distribution. It can be observed that \textsc{FedACS} performs comparably to most PFL methods, but with lower standard deviations. This similarity in performance may be attributed to the availability of sufficient data, allowing the global model to possess adequate generalization ability, while its local fine-tuning enables better adaptation to local data distributions. The abundance of training data simplifies the classification task for each client, as evidenced by the strong performance of models trained solely on local datasets. However, non-IID data settings pose challenges for global federated learning methods \cite{huang2021personalized}. \textsc{FedAvg}'s performance on Cifar10 declines noticeably because it introduces instability in the gradient-based optimization process. This instability arises from aggregating personalized models trained on non-IID data from different clients \cite{zhang2021fedpd}. \textsc{APFL} achieves the best performance in the pathological setting by adaptively learning the model and leveraging the relationship between local and global models, thus increasing the diversity of local models as learning progresses \cite{deng2020adaptive}. On the other hand, we can observe that \textsc{pFedMe} struggles to achieve good performance, as it focuses on solving the bilevel problem by using the aggregation of data from multiple clients to improve the global model rather than optimize local performance.

\subsection{Results on the insufficient non-IID data setting}
Figure \ref{fig:insuf} presents the mean accuracy and standard deviation per method and insufficient data setting. The two histograms on the left depict the mean accuracy with only 50 data samples per client, whereas those on the right showcase the results when the data is divided among 500 clients. It is evident that \textsc{FedACS} outperforms all other methods with smaller standard deviations in both the FMNIST and CIFAR10 datasets. The three local fine-tuning PFL methods face challenges across different datasets and settings, highlighting the fundamental bottleneck in PFL with non-IID data. By relying solely on a misleading global model, the significance of local data distribution in PFL is underestimated, and a single global model fails to capture the intricate pairwise collaboration relationships between clients. It is evident that \textsc{PerAvg} outperforms \textsc{pFedMe} in all settings. Both methods incorporate a ``meta-model" during training, but \textsc{PerAvg} utilizes it as an initialization to optimize a one-step gradient update for its personalized model. On the other hand, \textsc{pFedMe} simultaneously pursues both personalized and global models. As discussed in Section 2.3, leveraging the meta-model for a single local update step may yield better results when local devices have limited data. The reason behind \textsc{PerAvg}'s suboptimal performance in scenarios where each client possesses only 50 data samples can be attributed to its requirement of three data batches for a single update step. Consequently, the training process becomes insufficient due to the limited amount of data available. As discussed in Section 2.3, \textsc{FedAMP} fails to leverage the potential of the attention mechanism, resulting in underutilization. Additionally, it requires a large amount of local data to achieve local optima in each round, which is often unrealistic in many federated learning scenarios.

\subsection{Ablation study of data distribution}
To analyze the performance of \textsc{FedACS} with varying Dirichlet distribution parameter $\alpha$, we conducted experiments under conditions of limited data on two datasets, with each client retaining only 50 data samples for training. The results, encompassing both mean and variance of test accuracy, are presented in the table. As $\alpha$ increases, the non-IIDness in the data decreases, leading to a more uniform data distribution. This results in improved generalization capabilities for the global model. While the model performance exhibits enhancement, it is recommended to encourage users to involve a larger pool of participants in \textbf{FedACS} to further elevate the overall model performance.
\begin{table}[h]
	\caption{Test accuracy of global and personalized models.}
	\centering
	\begin{tabular}{l|ccccc}
		\toprule
		Dataset     &$\alpha = 0.1$   &$\alpha = 0.5$   &$\alpha = 1$   &$\alpha = 5$   &$\alpha = 10$ \\
		\midrule
		CIFAR10  &81.72$\pm$0.45  &83.84$\pm$0.19   &84.97$\pm$0.36  &85.76$\pm$0.14   &85.42$\pm$0.42 \\
            FMNIST &84.26$\pm$0.38	&84.33$\pm$0.35	&85.57$\pm$0.24  &86.13$\pm$0.34   &88.42$\pm$0.22 \\
		\bottomrule
	\end{tabular}
	\label{tab:ablation}
\end{table}

\section{Conclusion}
\label{sec:conclusion}
In this paper, we proposed \textsc{FedACS} to address the challenge of data heterogeneity and improve overall FL performance. Our approach leveraged the attention aggregation mechanism, allowing clients to identify and collaborate with those who can provide valuable information. This allowed \textsc{FedACS} to avoid being constrained by a single global model and instead optimize personalized models based on each client's local data distribution. We also analyzed the complexity of \textsc{FedACS} in achieving first-order optimality. Furthermore, we conducted a series of numerical experiments to demonstrate the superiority of \textsc{FedACS} in various non-IID settings. The results highlighted its effectiveness in handling data scarcity compared to other methods.

\newpage
\bibliography{reference}
\bibliographystyle{plainnat}

\newpage
\onecolumn
\appendix
\section*{Appendix}

\section{Missing Proofs}
\label{app-sec:proof}

\subsection{Proof of Theorem~\ref{theory:main}}

\maintheory*

\begin{proof}
The proof of the theorem is motivated by the analysis from Huang et al..\cite{huang2021personalized} Define the function $\mG_{L}:\R^{n\times d}\rightarrow \R$ as
\[\mG_{L}(W)=\min_{Z\in \R^{n\times d}} \fl (Z)+2L||Z-W||^2.\]
where $L$ is the Lipschitz constant of $\nabla \mR$. Also, here we denote $W_{*}$ to be the optimal solution of the minimization problem: $\mG_{L}(W)= \fl (W_{*})+2L||W_{*}-W||^2.$ In chapter 3.1, we defined the update of  $U^{k}$ , $U^{k}=W^{k-1}-\alpha_{k} \nabla \mR(W^{k-1})$. Then follow the updates:
\begin{equation*} \begin{split}
    \mG_{L}(U^{k}) &= \min_{Z} \fl (Z)+2L||Z-U^{k}||^2\\
      &\leq \fl(W^{k-1}_{*})+2L||W^{k-1}_{*}-U^{k}||^2     \\
      &= \fl(W^{k-1}_{*})+2L||W^{k-1}_{*}-W^{k-1}+\alpha_{k} \nabla \mR(W^{k-1})||^2 \\
      &= \fl(W^{k-1}_{*})+2L||W^{k-1}_{*}-W^{k-1}||^2+2\alpha_{k}^{2}L ||\nabla \mR(W^{k-1})||^2\\
      &\quad +4\alpha_{k} L\langle \nabla \mR(W^{k-1}), W^{k-1}_{*}-W^{k-1}\rangle\\
      &=\mG_{L}(W^{k-1})+2\alpha_{k}^{2}L ||\nabla \mR(W^{k-1})||^2 +4\alpha_{k} L\langle \nabla \mR(W^{k-1}), W^{k-1}_{*}-W^{k-1}\rangle.\\
\end{split} \end{equation*}\\

Following Assumption 1, we have $||\nabla \mR(W^{k-1})||\leq B/\lambda$. Besides, from the Lipschitz continuous property of $\nabla \mR$, it holds that

\begin{equation*}
    \mR(W^{k-1}_{*})-\mR(W^{k-1})-\langle \nabla \mR(W^{k-1}), W^{k-1}_{*}-W^{k-1}\rangle\\
    \geq -\frac{L}{2\lambda}||W^{k-1}_{*}-W^{k-1}||^2,
\end{equation*}

Thus, we have

\begin{equation}\begin{split}
    \mG_{L}(U^{k}) &\leq \mG_{L}(W^{k-1})+4\alpha_{k} L( \mR(W^{k-1}_{*})-\mR(W^{k-1}))\\
    &\quad+\frac{2\alpha_{k}L^2}{\lambda}||W^{k-1}_{*}-W^{k-1}||^2+\frac{2\alpha_{k}LB^2}{\lambda^2}. \label{GU}
\end{split} \end{equation}

Moreover, the update of $W^{k}$: $W^{k}=U^{k}-\beta_{k} \nabla \mF(U^{k})$ together with the definition of $\mG_L$, yields
\begin{equation*} \begin{split}
    \mG_{L}(W^{k}) &= \min_{Z} \fl (Z)+2L||Z-W^{k}||^2\\
      &\leq \fl(U^{k}_{*})+2L||U^{k}_{*}-W^{k}||^2     \\
      &= \fl(U^{k}_{*})+2L||U^{k}_{*}-U^{k}+\beta_{k} \nabla \mF(U^{k})||^2 \\
      &= \fl(U^{k}_{*})+2L||U^{k}_{*}-U^{k}||^2+2\beta_{k}^{2}L ||\nabla \mF(U^{k})||^2\\
      &\quad +4\beta_{k} L\langle \nabla \mF(U^{k}), U^{k}_{*}-U^{k}\rangle\\
      &=\mG_{L}(U^{k})+2\beta_{k}^{2}L ||\nabla \mF(U^{k})||^2+4\beta_{k} L\langle \nabla \mF(U^{k}), U^{k}_{*}-U^{k}\rangle.\\
\end{split} \end{equation*}\\

Following Assumption 1, we have $||\nabla \mF(W^{k})||\leq B$. Besides, from the Lipschitz continuous property of $\nabla \mF$, it holds that
\begin{equation*}
    \mF(U^{k}_{*})-\mF(U^{k})-\langle \nabla \mF(U^{k}), U^{k}_{*}-U^{k}\rangle\\
    \geq -\frac{L}{2}||U^{k}_{*}-U^{k}||^2,
\end{equation*}
Thus, we have

\begin{equation}\begin{split}
    \mG_{L}(W^{k}) &\leq \mG_{L}(U^{k})+4\beta_{k} L( \mF(U^{k}_{*})-\mF(U^{k}))\\
    &\quad+2\beta_{k}L^2+2\beta_{k}LB^2. \label{GW}
\end{split} \end{equation}

Since $U^{k}_{*}$ is the optimal solution of $\mG_{L}(U^{k})$, we have
\begin{equation*}
    \nabla \fl(U^{k}_{*})+4L(U^{k}_{*}-U^{k})=0, 
\end{equation*}
which implies
\begin{equation}
    ||U^{k}_{*}-U^{k}||^2\leq \frac{B^2}{4L^2}.\label{UU}
\end{equation}

From Huang et al.,\cite{huang2021personalized} we claim that

\begin{equation}
    ||W^{k-1}_{*}-U^{k}_{*}||\leq 2||W^{k-1}- U^{k}|| \leq \frac{2\alpha_{k}B}{\lambda}. \label{WU}
\end{equation}
Then, by \eqref{WU}, the Lipschitz continuity of $\nabla \mF$, and Assumption 1, we derive
\begin{equation*} \begin{split}
    \mF(U^{k}_{*})-\mF(U^{k})&=\mF(U^{k}_{*})-\mF(W^{k-1}_{*)}+\mF(W^{k-1}_{*})-\mF(W^{k-1})+\mF(W^{k-1})-\mF(U^{k})\\
    &\leq \mF(W^{k-1}_{*})-\mF(W^{k-1})+B||W^{k-1}_{*}-U^{k}_{*}||+B||W^{k-1}- U^{k}||\\
    &\leq \mF(W^{k-1}_{*})-\mF(W^{k-1})+\frac{3\alpha_{k}B^2}{\lambda}.\\
\end{split} \end{equation*}

When $\beta_{k}=\alpha_{k}/\lambda$, combining the above equation with \eqref{GW} and \eqref{UU} gives us
\begin{equation}
    \mG_{L}(W^{k}) \leq \mG_{L}(U^{k})+\frac{4\alpha_{k}L}{\lambda}( \mF(W^{k-1}_{*})-\mF(W^{k-1}))
    +\frac{14\alpha_{k}^2LB^2}{\lambda^2}+\frac{\alpha_{k}B^2}{2\lambda}. \label{GWU1}
\end{equation}

Upon adding \eqref{GWU1} with \eqref{GW} and combining $\fl=\mF+\lambda\mR $, we get
\begin{equation}\begin{split}
    \mG_{L}(W^{k}) &\leq \mG_{L}(W^{k-1})+\frac{4\alpha_{k}L}{\lambda}( \fl(W^{k-1}_{*})-\fl(W^{k-1}))+\frac{2\alpha_{k}L^2}{\lambda}||W^{k-1}_{*}-W^{k-1}||^2+\frac{16\alpha_{k}^2LB^2}{\lambda^2}+\frac{\alpha_{k}B^2}{2\lambda}\\
    &=\mG_{L}(W^{k-1})+\frac{4\alpha_{k}L}{\lambda}( \fl(W^{k-1}_{*})-\fl(W^{k-1})+2L||W^{k-1}_{*}-W^{k-1}||^2)\\
    &\quad -\frac{6\alpha_{k}L^2}{\lambda}||W^{k-1}_{*}-W^{k-1}||^2+\frac{16\alpha_{k}^2LB^2}{\lambda^2}+\frac{\alpha_{k}B^2}{2\lambda}\\
    &\leq \mG_{L}(W^{k-1})-\frac{6\alpha_{k}L^2}{\lambda}||W^{k-1}_{*}-W^{k-1}||^2+\frac{16\alpha_{k}^2LB^2}{\lambda^2}+\frac{\alpha_{k}B^2}{2\lambda}.\\ \label{GWW}
\end{split} \end{equation}
The second inequality is from the definition of $W^{k-1}_{*}$ that $\fl(W^{k-1}_{*})+2L||W^{k-1}_{*}-W^{k-1}||^2\leq \fl(W^{k-1})$.

We now consider the case where $\alpha_{k}=\lambda/\sqrt{K}=\alpha$ for all $k=1,2,...,K.$ By summing up  \eqref{GWW} from $k=1$ to $k=K$, we obtain
\begin{equation}
    \min_{0\leq k \leq K}||W^{k}_{*}-W^{k}||^2 \leq \frac{1}{K}\sum^{K}_{k=1}||W^{k-1}_{*}-W^{k-1}||^2
    \leq \frac{\lambda}{6\alpha L^2}\frac{\mG_{L}(W^{0})-\mG_{L}(W^{k})}{K}+\frac{8\alpha B^2}{3\lambda L}+\frac{B^2}{12KL^2}. \label{minWW}
\end{equation}
From the definition of $\mG_{L}$, we can verify that $\mG_{L}(W^{0})\leq \fl(W^{0})$, and $\mG_{L}(W^{0})\geq \fl^{*},$ where $\fl^{*}$ is the optimal value of $\fl$. Also, using the definition of $W^{k}_{*}$ and taking the fact that $\nabla \fl$ is Lipschitz continuous with the constant 2$L$, we can obtain that
\begin{equation*}
    \begin{split}
        ||\nabla \fl(W^{k})||&\leq ||\nabla \fl(W^{k}_{*})||+||\nabla \fl(W^{k})_{*}-\nabla \fl(W^{k})||\\
        &\leq 6L||W^{k}_{*}-W^{k}||.\\
    \end{split}
\end{equation*}
Combining this, \eqref{minWW}, and $\alpha=\lambda/\sqrt{K}$, we have
\begin{equation}
    \min_{0\leq k \leq K}||\nabla \fl(W^{k})||^2\leq \frac{6( \fl(W^{0})-\fl^{*}+16LB^2)}{\sqrt{K}}+\mathcal{O}\bigg( \frac{1}{K}\bigg)
\end{equation}

besides, we can bound $||U^{k}_{*}-U^{k}||^2$ as
\begin{equation*}\begin{split}
    ||U^{k}_{*}-U^{k}||^2&=||U^{k}_{*}-W^{k-1}_{*}+W^{k-1}_{*}-W^{k-1}+W^{k-1}-U^{k}||^2\\
    &\leq 4||U^{k}_{*}-W^{k-1}_{*}||^2+2||W^{k-1}_{*}-W^{k-1}||^2+4||W^{k-1}-U^{k}||^2\\
    &\leq \frac{20\alpha_{k}^2B^2}{\lambda^2}+2||W^{k-1}_{*}-W^{k-1}||^2
\end{split}\end{equation*}

When $\beta_k=\alpha_k/\lambda$, combining the above equations with \eqref{GW} gives us
\begin{equation}
    \mG_{L}(W^{k}) \leq \mG_{L}(U^{k})+\frac{4\alpha_{k}L}{\lambda}( \mF(W^{k-1}_{*})-\mF(W^{k-1}))
    +\frac{4\alpha_{k}L^2}{\lambda}||W^{k-1}_{*}-W^{k-1}||^2
    +\frac{14\alpha_{k}^2LB^2}{\lambda^2}+\frac{40\alpha_{k}^3L^2B^2}{\lambda^3}. \label{GWU2}
\end{equation}

Upon adding \eqref{GWU2} with \eqref{GW} and combining $\fl=\mF+\lambda\mR $, we get
\begin{equation}
    \mG_{L}(W^{k}) \leq \mG_{L}(W^{k-1})-\frac{2\alpha_{k}L^2}{\lambda}||W^{k-1}_{*}-W^{k-1}||^2+\frac{16\alpha_{k}^2LB^2}{\lambda^2}+\frac{40\alpha_{k}^3L^2B^2}{\lambda^3}.
 \end{equation}
Summing up the above inequality from $k=1$, we have
\begin{equation*}
    \sum^{\infty}_{k=1} \alpha_{k} ||W^{k-1}_{*}-W^{k-1}||^2\leq \frac{\lambda}{2L^2}\mG_{L}(W^{0})+
    \frac{7B^2}{\lambda L}\sum^{\infty}_{k=1} \alpha_{k}^2 +\frac{20B^2}{\lambda^2}\sum^{\infty}_{k=1} \alpha_{k}^3.
\end{equation*}
Then, for the case where $\sum^{\infty}_{k=1} \alpha_{k}=\infty$ and $\sum^{\infty}_{k=1} \alpha_{k}^2<\infty$ , we obtain
\begin{equation*}
    \sum^{\infty}_{k=1} \alpha_{k} ||W^{k-1}_{*}-W^{k-1}||^2 < \infty,
\end{equation*}
which yields that 
\begin{equation*}
    \liminf_{k\rightarrow \infty}||W^{k-1}_{*}-W^{k-1}||=0.
\end{equation*}
% The results of the theorem are proved as desired.
This completes the proof.
\end{proof}

\end{document}